\documentclass[twoside,11pt]{article}

\usepackage{amsmath,amsfonts,amssymb}

\usepackage{jmlr2e}
\usepackage{ mathrsfs }

\usepackage{algorithm}
\usepackage{algorithmic}

\ShortHeadings{On Consistency of Compressive Spectral Clustering}{Muni Sreenivas Pydi and Ambedkar Dukkipati}
\firstpageno{1}

\begin{document}

\title{On Consistency of Compressive Spectral Clustering}

\author{\name Muni Sreenivas Pydi \email pydi@wisc.edu\\
		\addr Department of Electrical and Computer Engineering\\
		University of Wisconsin - Madison\\
		Madison, WI - 53726, USA
       \AND
       \name Ambedkar Dukkipati \email ambedkar@iisc.ac.in\\
       \addr Department of Computer Science and Automation\\
       Indian Institute of Science\\
       Bengaluru - 560012, India}

\maketitle

\begin{abstract}
Spectral clustering is one of the most popular methods for community
detection in graphs. A key step in spectral clustering algorithms is
the eigen decomposition of the $n{\times}n$ graph Laplacian matrix to extract its $k$
leading eigenvectors, where $k$ is the desired number of clusters among $n$
objects. This is prohibitively complex to implement for
very large datasets. However, it has recently been shown that it is possible
to bypass the eigen decomposition by computing an approximate spectral embedding through
graph filtering of random signals. In this paper, 
we analyze the working of spectral clustering performed via graph filtering on the
stochastic block model. Specifically, we characterize the 
effects of sparsity, dimensionality and filter
approximation error on the consistency of the algorithm in recovering planted clusters.

\end{abstract}

\begin{keywords}
  spectral methods, clustering, stochastic block model
\end{keywords}

\section{Introduction}

Detecting communities, or clusters in networks is an important problem in many
fields of science \citep{fortunato2010community, jain1999data}. Spectral clustering
is a widely used algorithm for community detection in networks \citep{von2007tutorial}
because of its strong theoretical grounding \citep{ng2002spectral,shi2000normalized}
and recently established consistency results \citep{rohe2011spectral,lei2015consistency}.
Spectral clustering works by relaxing the NP-hard discrete optimization problem of
graph partitioning, into a continuous optimization problem. As a first step,
one computes the the $k$ leading eigenvectors of the graph Laplacian matrix,
that gives a $k$ dimensional 'spectral' embedding for each vertex of the graph.
In the second step, one performs $k$-means on the embedding to
retrieve the graph clusters.

However, computing the leading eigenvectors of the graph Laplacian requires eigen decomposition, which is very hard to compute
for large datasets. Several approximate algorithms have been proposed
to overcome this problem via Nystr{\"o}m sampling
\citep{fowlkes2004spectral, Li2011TimeAS, choromanska2013fast}.
While these methods do not skip the eigen decomposition, 
they reduce its complexity via column sampling of the Laplacian. 
Another class of methods use random projections to reduce the 
dimensionality of the dataset while obtaining an approximate 
spectral embedding \citep{sakai2009fast, gittens2013approximate}. 
On the other hand, with the emergence of signal processing on graphs 
\citep{shuman2013emerging}, there has been the development of 
techniques based on graph filtering that can side-step the 
eigen decomposition altogether \citep{ramasamy2015compressive, 
tremblay2016compressive, tremblay2016accelerated}. While many of 
these approaches have been shown to work fairly well on real 
and synthetic datasets, a rigorous mathematical analysis is 
still lacking.

In this paper, we consider a variant of the compressive spectral
clustering algorithm that uses graph filtering of random signals to compute an approximate
spectral embedding of the graph nodes~\citep{tremblay2016compressive}. For a graph with $n$
nodes and $k$ clusters, the algorithm proceeds by calculating a $d$
dimensional embedding for the graph nodes, where $d$ is of
the order of $\log(n)$. This {\em compressed} embedding acts as a
substitute for the $k$ dimensional spectral embedding of the spectral
clustering algorithm and does not need the eigen decomposition of the
Laplacian. Instead, the embedding is obtained by filtering out the top
$k$ frequencies for $d$ number of random graph signals using fast
graph filtering. 

\subsection*{Contributions}

In this paper, we analyze the spectral clustering algorithm performed via graph filtering (Algorithm \ref{SC-GF_alg})
using the stochastic block model (SBM). We derive a bound on the number of vertices that would be 
incorrectly clustered with the algorithm, and prove that the algorithm can consistently recover planted clusters from SBM under mild assumptions on the sparsity of the graph and the filter approximation used to compute the spectral embedding.
For our analysis, we specifically consider the high-dimensional
stochastic block model that allows for the number of clusters $k$ to
grow faster than $\log(n)$. This is very important considering that the computational gains
of compressive spectral algorithm is more apparent in the high-dimensional case. In proving the weak consistency 
of Algorithm \ref{SC-GF_alg}, we primarily use the proof techniques
from~\cite{rohe2011spectral}, which were originally used to analyze
the spectral clustering algorithm under the high-dimensional SBM. 
Finally, we analyze our consistency result in some special cases of the block model and validate our
findings with accompanying experiments.



\section{Preliminaries}
\label{sec_Preliminaries}

\subsection{Notation}
We use capital letters to denote matrices, and specifically their formal script versions for random
matrices. We use the superscript $^{(n)}$ to denote matrices corresponding to a graph of $n$ nodes. We use $\lVert \cdot \rVert_2$ for the Euclidean norm of a vector and the spectral norm
of a matrix. We use $\lVert \cdot \rVert_F$ for the Frobenius norm of a matrix. For a matrix  $M$,
we use $M_{i*}$ and $M_{*j}$ to denote the $i$th row and $j$th column respectively. We also use the standard
notation $o(\cdot)$, $O(\cdot)$ and $\Omega(\cdot)$ to describe the limiting behavior of functions.

\subsection{Stochastic Block Model}

We consider an undirected, unweighted graph $G$ with $n$ nodes. Under SBM, each node of the graph $G$ is assigned to one of $k$ clusters or blocks via the membership matrix $Z \in \{0, 1\}^{n \times k}$. $Z_{ig}=1$ if and only if the node $i$ belongs to block $g$. The SBM adjacency matrix is defined as
$\mathscr{W} = ZBZ^{T}$
where $B \in [0, 1]^{k \times k}$ is the block matrix, whose entry $B_{gh}$ gives the probability of an edge between nodes of cluster $g$ and cluster $h$. $B$ is full rank and symmetric. The diagonal entries of $\mathscr{W}$ are set to zero to prevent self edges. From $\mathscr{W}$, we define the degree matrix $\mathscr{D}$ such that $\mathscr{D}_{ii} = \sum_{k}\mathscr{W}_{ik}$ and  the normalized Laplacian matrix $\mathscr{L} = \mathscr{D}^{-1/2} \mathscr{W} \mathscr{D}^{-1/2}$.
We define $\tau_n = \min_{1 \leq i \leq n} \mathscr{D}^{(n)}_{ii}/n$ to indicate the level of sparsity in the graph.

To generate a random graph with SBM, we sample a random adjacency
matrix $W$ from it's population version, $\mathscr{W}$. Let $D$ and
$L$ represent the corresponding degree matrix and the normalized
Laplacian for the sampled graph. Using Davis-Kahan theorem, it can be
shown that the eigenvectors of $L$ and $\mathscr{L}$ converge
asymptotically as $n$ becomes large. This is important because the
spectral clustering algorithm relies on the eigenvectors of the
sampled graph Laplacian $L$ to estimate the node membership $Z$. 

Now, we borrow a result from \cite{rohe2011spectral} that shows the conditions for convergence of the leading $k$ eigenvectors of $L$ and $\mathscr{L}$.

\begin{theorem}[Convergence of Eigenvalues and Eigenvectors]\label{Th_conv}

Let $W^{(n)} \in \{0, 1\}^{n \times n}$ be a sequence of adjacency matrices sampled from the SBM with population matrices $\mathscr{W}^{(n)}$. Let $L^{(n)}$ and $\mathscr{L}^{(n)}$ be the corresponding graph Laplacians. Let $X^{(n)}, \mathcal{X}^{(n)} \in \mathbb{R}^{n \times k_n}$ be the matrices that contain the eigenvectors corresponding to the leading $k_n$ eigenvalues of $L^{(n)}$ and $\mathscr{L}^{(n)}$ in absolute sense, respectively. Let $\bar{\lambda}_{k_n}$ be the least non-zero eigenvalue of $\mathscr{L}^{(n)}$.
\begin{assumption}[Eigengap]\label{Th_conv_assump1}
$n^{-1/2}(\log n)^2 = \mathcal{O}(\bar{\lambda}_{k_n}^2)$
\end{assumption}

\begin{assumption}[Sparsity]\label{Th_conv_assump2}
$\tau_n^2 > 2/\log n$
\end{assumption}
Under Assumptions \ref{Th_conv_assump1} and \ref{Th_conv_assump2}, for some sequence of orthonormal matrices $O^{(n)}$,
\begin{equation*}
\lVert X^{(n)} - \mathcal{X}^{(n)} O^{(n)} \rVert_F^2 = o\bigg(\frac{(\log n)^2}{n \bar{\lambda}_{k_n}^4 \tau_n^4}\bigg).
\end{equation*}
\end{theorem}

\begin{proof}
Theorem \ref{Th_conv} is a special case of Theorem 2.2 from \citet{rohe2011spectral}. The result follows by setting $S_n = [\bar{\lambda^2_{k_n}}/2, 1]$ and $\delta_n = \delta^{\prime}_n = \bar{\lambda}^2_{k_n}/2$.
\end{proof}
While Assumption \ref{Th_conv_assump1} ensures that the eigengap of $\mathscr{L}$ is high enough to enable the separability of the $k$ clusters, Assumption \ref{Th_conv_assump2} puts a lower bound on the sparsity level of the graph. Under these two assumptions, Theorem \ref{Th_conv} bounds the Frobenius norm of the difference between the top $k$ eigenvectors of the population and sampled versions of the graph Laplacian.

\subsection{Spectral Clustering}

Spectral clustering operates on the $k$ leading eigenvectors of $L$ i.e. the matrix $X$ in Theorem \ref{Th_conv}. Each row of $X$ is taken as the $k$-dimensional spectral embedding of the corresponding node, and $k$-means is performed on the new data points to retrieve the cluster membership matrix $Z$. The spectral clustering algorithm we consider is listed in Algorithm \ref{SC_alg}.

Note that the $k$-means is performed on the rows of the matrix $X$ in
Algorithm \ref{SC_alg}. For this to result in the $k$ distinct
clusters of the SBM, the rows belonging to nodes in different clusters
must be \lq{well-separated}\rq \ while the rows belonging to nodes in
the same cluster must be closely spaced. This property of $X$ becomes
evident from Theorem~\ref{Th_sep} that follows from the work of \cite{rohe2011spectral}.

\begin{algorithm}[tb]\label{SC}
   \caption{Spectral Clustering}
   \label{SC_alg}
\begin{algorithmic}
   \STATE {\bfseries Input:} Graph Laplacian matrix $L$, number of clusters $k$
   \STATE 1. Compute $X \in \mathbb{R}^{n \times k}$ containing the eigenvectors corresponding to $k$ leading eigenvalues (in absolute sense) of $L$.
   \STATE 2. Treating each row of  $X$ as a point in $\mathbb{R}^{k}$, run $k$-means. From the result of $k$-means, form the membership matrix $\hat{Z} \in \{0, 1\}^{n \times k}$ assigning each node to a cluster.
   \STATE {\bfseries Output:} Estimated membership matrix $\hat{Z}$.
\end{algorithmic}
\end{algorithm}

\begin{theorem}[Separability of Clusters]\label{Th_sep}
Consider a SBM with $k$ blocks. Let $\mathscr{L}$ be the population version of the graph Laplacian. Let $\mathcal{X} \in \mathbb{R}^{n \times k}$ be the matrix containing the eigenvectors corresponding to $k$ nonzero eigenvalues of $\mathscr{L}$. Let $P$ be the number of nodes in the largest block i.e. $P = max_{1 \leq j \leq k} (Z^{T}Z)_{jj}$ Then the following statements are true.
\begin{enumerate}
\item There exists a matrix $\mu \in \mathbb{R}^{k \times k}$ such that $Z\mu = \mathcal{X}$.
\item $\mathcal{X}_{i*} = \mathcal{X}_{j*} \Leftrightarrow Z_{i*} = Z_{j*}$ i.e. $\mu$ is invertible.
\item $\lVert \mathcal{X}_{i*} - \mathcal{X}_{j*} \rVert_2 \geq \sqrt{2/P}$ for any $Z_{i*} \neq Z_{j*}$.
\end{enumerate}

\end{theorem}

\begin{proof}
Statements 1 and 2 of Theorem \ref{Th_sep} follow from Lemma 3.1 from \citet{rohe2011spectral}. Statement 3 is equivalent to Statement D.3 from the proof of Lemma 3.2 in \citet{rohe2011spectral}.
\end{proof}

From Theorem \ref{Th_sep}, it is evident that performing $k$-means on the rows of $\mathcal{X}$ would retrieve the block membership of all the nodes in the graph exactly. However, the matrix $\mathcal{X}$ is hidden, and only its sampled version, $X$ can be accessed. But by theorem \ref{Th_conv}, we have that $X$ is a close approximation of $\mathcal{X}$ for large $n$. As Algorithm \ref{SC_alg} performs $k$-means on $X$, the estimated membership matrix $\hat{Z}$ should be close to the true membership matrix $Z$.

\subsection{Graph Filtering}\label{subsec_GraphFiltering}

As in Algorithm \ref{SC_alg}, extracting the top $k$ eigenvectors of the Laplacian is a key step in the spectral clustering algorithm. This can be viewed as extracting the $k$ lowest frequencies or Fourier modes of the graph Laplacian. This interpretation allows us to use the fast graph filtering approach \citep{tremblay2016compressive, ramasamy2015compressive} to speed up the computation. We briefly describe
this here.

A graph signal $y \in \mathbb{R}^n$ is a mapping from vertex set $V$ of a graph $G$ to $\mathbb{R}$. If the eigen decomposition of the graph Laplacian is $L = U \Lambda U^T$, then the graph Fourier transform of $y$ is $\hat{y} = U^{T}y$. The entries of $\hat{y}$ give the $n$ Fourier modes of the graph signal $y$. Assuming that the rows of $U$ are ordered in the decreasing order (in absolute value) of the corresponding eigenvalues, the top $k$ Fourier modes of $y$ can be obtained by $\hat{y}_k = X^{T}y$ where $X \in \mathbb{R}^{n \times k}$ is the matrix whose
columns are the top $k$ eigenvectors of $L$.

A graph filter function $h$ is defined over $[-1, 1]$, the range of eigenvalues of the normalized graph Laplacian. The filter operator in the graph domain, $h(\Lambda)$ is a diagonal matrix defined as $h(\Lambda) := \mathrm{diag}(h(\lambda_1), \ldots, h(\lambda_n))$ where $\lambda_1, \ldots, \lambda_n$ are the eigenvalues of $L$ ordered in the decreasing order of absolute value. The equivalent filter operator in the spectral domain, $H \in \mathbb{R}^{n \times n}$ is defined as $H := U h(\Lambda) U^{T}$.

To extract the top $k$ Fourier modes of a graph signal, we use an ideal low-pass filter defined as
\begin{equation}\label{IdealFilter}
h_{\lambda_k}(\lambda) = 
\begin{cases}
1 \ \  if |\lambda| \geq |\lambda_k| \\
0 \ \  otherwise.
\end{cases}
\end{equation}

The result of graph signal $y$ filtered through $h_{\lambda_k}$ is given by $y_{\lambda_k} = U h_{\lambda_k}(\Lambda) U^T y = XX^T y$. Obviously, filtering a graph signal with the ideal filter in \eqref{IdealFilter} needs the eigen decomposition of the graph Laplacian. Now we define $\widetilde{h}_{\lambda_k}(\lambda) := \sum_{{\ell}=0}^p \alpha_{\ell} \lambda^{\ell}$, an order $p$ polynomial, to be the non-ideal approximation of the filter $h_{\lambda_k}(\lambda)$. The filter operator in spectral domain, $\widetilde{H}_{\lambda_k}$ can be computed as
$\widetilde{H}_{\lambda_k} = U \widetilde{h}_{\lambda_k}(\lambda) U^T = \sum_{{\ell}=0}^p \alpha_{\ell} L^{\ell}$.
The signal $y$ filtered by $\widetilde{H}_{\lambda_k}$ can be computed as $\widetilde{y}_k = \sum_{{\ell}=0}^p \alpha_{\ell} L^{\ell} y$, which does not required the eigen decomposition of $L$. Moreover, it only involves computing $p$ matrix-vector multiplications.

The method that we use for our analysis is outlined in Algorithm~\ref{SC-GF_alg}.

\begin{algorithm}[tb]\label{Alg_SC-GF}
   \caption{Spectral Clustering via Graph Filtering}
   \label{SC-GF_alg}
\begin{algorithmic}
   \STATE {\bfseries Input:} Graph Laplacian $L$, number of clusters $k$, number of dimensions $d$, polynomial order $p$.
   \STATE 1. Estimate $\lambda_k$ of $L$.
   \STATE 2. Compute $\widetilde{h}_{\lambda_k}$ to approximate the ideal filter $h_{\lambda_k}$.
   \STATE 3. Construct $R \in \mathbb{R}^{n \times d}$ with i.i.d entries from $\mathcal{N}(0,\frac{1}{d})$.
   \STATE 4. Compute  $\widetilde{X}_R = \widetilde{H}_{\lambda_k} R = \sum_{l=0}^p \alpha_l L^l R$.
   \STATE 5. Treating each row of  $\widetilde{X}_R$ as a point in $\mathbb{R}^{d}$, run $k$-means. From the result of $k$-means, form the membership matrix $\hat{Z} \in \{0, 1\}^{n \times k}$ assigning each node to a cluster.
   \STATE {\bfseries Output:} Estimated membership matrix $\hat{Z}$.
\end{algorithmic}
\end{algorithm}

\section{SBM and Spectral Clustering via Graph Filtering}
\label{sec_SCviaGF}

In this section, we lay down the building blocks that make up Algorithm~\ref{SC-GF_alg}. In \ref{CSpecEmb} we shall see how a compressed spectral embedding can be computed with graph filtering and prove that the compressed embedding is still a close approximation of the SBM's population version of graph Laplacian. In Section \ref{sec_FGF} we show the effect of using the fast graph filtering technique to compute the compressed embedding. In Section \ref{Sec_LambdakEst} we deal with the estimation of the $k^{th}$ eigenvalue of the
graph Laplacian without resorting to eigen decomposition.

\subsection{Compressed Spectral Embedding}\label{CSpecEmb}
From Algorithm \ref{SC_alg}, it seems that we need the matrix $X$ containing the $k$ most significant eigenvectors of $L$, in order to retrieve the clusters. Since we only use the rows of $X$ as data points for the subsequent $k$-means step, we only need a distance preserving embedding of the rows of $X$. In this section, we see how such an embedding can be obtained through the result of filtering random graph signals. The technique used is similar to that of \cite{tremblay2016compressive}, except that we employ stricter assumptions to help in proving consistency results.

Consider the matrix $R \in \mathbb{R}^{n \times d}$ whose entries are independent Gaussian random variables with mean $0$ and variance $1/d$. Define $X_R := H_{\lambda_k} R = U h_{\lambda_k}(\Lambda) U^T R = XX^T R$ whose $d$ columns contain the result of filtering the corresponding $d$ columns of $R$ using the filter $h_{\lambda_k}$. In Theorem \ref{Th_compressedconvSep}, we show that the rows of $X_R$ form an $\epsilon$-approximate distance preserving embedding of the rows of $X$ for sufficiently large $d$. To analyze the effect of this embedding on the true cluster centers, i.e. the $k$ unique rows of $\mathcal{X}$, we define the matrix $\mathcal{X}_R := \mathcal{X}O X^T R$ where $O$ is the orthonormal rotation matrix as in Theorem \ref{Th_conv}. We aim to show that the separability of the true cluster centers is still ensured under the compressed embedding.

\begin{theorem}[Convergence and Separability under Compressed Spectral Embedding]\label{Th_compressedconvSep}
For the sequence of adjacency matrices as defined in Theorem \ref{Th_conv}, define $P_n = max_{1 \leq j \leq k_n} (Z^{T}Z)_{jj}$ to be the sequence of populations of the largest block. Let $X^{(n)}_R , \mathcal{X}^{(n)}_R \in \mathbb{R}^{n \times d_n}$ be the compressed embeddings for $X^{(n)}, \mathcal{X}^{(n)} \in \mathbb{R}^{n \times k_n}$ as defined in Theorem \ref{Th_conv}.
For $\epsilon_1\in[0,1]$ and $\beta>0$, if
\begin{equation*}
d_n > \frac{4+2\beta}{\epsilon_1^2/2-\epsilon_1^3/3}\log (n+k_n),
\end{equation*}
then with probability at least $1-n^{-\beta}$, we have the following under the assumptions of Theorem \ref{Th_conv}.
\begin{align*}
&\lVert X^{(n)}_R - \mathcal{X}^{(n)}_R \rVert_F^2 = o\bigg(\frac{(\log n)^2}{n \bar{\lambda}_{k_n}^4 \tau_n^4}\bigg),\\
&\lVert {\mathcal{X}_R}_{i*}^{(n)} -{\mathcal{X}_R}_{j*}^{(n)} \rVert_2 \geq (1-\epsilon_1)\sqrt{2/P_n}
\end{align*}
for any $Z_{i*} \neq Z_{j*}$, where $i,j \in \{1, \cdots, n\}$.

\end{theorem}

\begin{proof}
See Appendix \ref{appB}.
\end{proof}
Theorem \ref{Th_compressedconvSep} is analogous to the theorems on convergence (Theorem \ref{Th_conv})
and separability (Theorem \ref{Th_sep}) of the spectral clustering algorithm. It ensures that the approximate
spectral embedding $X_R^{(n)}$ converges to the corresponding population version $\mathcal{X}^{(n)}_R$
while still ensuring that the true clusters remain separable.

\subsection{Efficient Computation via Fast Graph Filtering}\label{sec_FGF}

Now, we define an additional level of approximation for the spectral embedding using the fast graph filtering technique discussed in Section \ref{subsec_GraphFiltering}. 
Let $\widetilde{X}_R := \widetilde{H}_{\lambda_k} R = \sum_{{\ell}=0}^p \alpha_{\ell} L^{\ell} R$ to be the output of approximate filtering of the columns of $R$ where $R \in \mathbb{R}^{n \times d}$ with entries drawn from $\mathcal{N}(0, \frac{1}{d})$. Lemma \ref{Lem_ApproxFil} bounds the difference between $\widetilde{X}_R$ and $X_R$, which result from approximate and ideal filtering respectively.

\begin{lemma}[Bounding the Approximate Filtering Error]\label{Lem_ApproxFil}
For the sequence of adjacency matrices as defined in Theorem \ref{Th_conv}, let $\widetilde{X}^{(n)}_R \in \mathbb{R}^{n \times d_n}$ be the approximation for ${X}^{(n)}_R$ obtained using the polynomial filter $\widetilde{h}_{\lambda_{k_n}}$ (instead of the ideal filter $h_{\lambda_{k_n}}$). Let $\sigma(L^{(n)})$ be the spectrum of the sampled graph Laplacian $L^{(n)}$. Define the maximum absolute error in the polynomial approximation as
$e_n = \max_{\lambda \in \sigma(L^{(n)})} | \widetilde{h}_{\lambda_{k_n}}(\lambda) - h_{\lambda_{k_n}}(\lambda) |$.
For $\epsilon_2 \in[0,1]$, with probability at least $1-e^{-n d_n(\epsilon_2^2 - \epsilon_2^3)/4}$,
\begin{equation*}
\lVert \widetilde{X}_R^{(n)} - X_R^{(n)} \rVert_F^2 \leq (1+\epsilon_2) n^2 e_n^2.
\end{equation*}
\end{lemma}

\begin{proof}
See Appendix \ref{appB}.
\end{proof}

\subsection{Estimation of $\lambda_k$}\label{Sec_LambdakEst}

Lemma \ref{Lem_ApproxFil} shows that in order to achieve a fixed error bound between $X_R$ and $\widetilde{X}_R$, the polynomial approximation must be increasingly accurate  as $n$ grows large. Designing such a polynomial would necessitate knowing the value of $\lambda_k$. In this section, we explain how that can be done without having to do the eigen decomposition of $L$. First, we state the following Lemma which bounds 
the output of fast graph filtering, $\widetilde{X}_R$.

\begin{lemma}[Estimation of $\lambda_k$]\label{Lem_lambda_k}
For the $\widetilde{X}_R^{(n)}$, $e_n$ and $\epsilon_2$ given in Lemma \ref{Lem_ApproxFil}, with probability at least $1-e^{-n d_n(\epsilon_2^2 - \epsilon_2^3)/4}$ we have
\begin{align*}
(1-\epsilon_2)k_n - 2(1+\epsilon_2)k_n e_n 
\leq \frac{1}{n} \lVert \widetilde{X}_R^{(n)} \rVert_F^2 
\leq (1+\epsilon_2)(k_n + 2k_n e_n + ne_n^2).
\end{align*}
\end{lemma}

\begin{proof}
See Appendix \ref{appB}.
\end{proof}
For $(2ke_n + ne_n^2) = o(1)$, Lemma \ref{Lem_lambda_k} shows that the output of fast graph filtering, $\widetilde{X}_R$ is tightly concentrated around $k$, upon normalization by $n$. This can be used to estimate  $|\lambda_k|$ by a dichotomic search in the range $[0,1]$ as explained in \cite{puy2016random}. The basic idea is to make a coarse initial guess on $|\lambda_k|$ in the interval $[0,1]$, compute $\widetilde{X}_R$ with the current estimate, and iteratively refine the estimate by comparing $\frac{1}{n} \lVert \widetilde{X}_R \rVert_F^2$ with $k$.

Before we move on to proving the consistency of Algorithm \ref{SC-GF_alg}, let us summarise the results from the previous sections. We have a tractable way to estimate $|\lambda_k|$ without the eigen decomposition of $L$. Through Lemma \ref{Lem_ApproxFil}, we know that the resultant approximate embedding will be close to the ideal compressed embedding, for reasonably accurate polynomial approximation of the ideal filter. Through Theorem \ref{Th_compressedconvSep}, we showed that a compressed embedding of the $k$ leading eigenvectors of $L$ converge to the corresponding embedding on $\mathscr{L}$. We also showed that the data points corresponding to different clusters are still separable under such an embedding.

\section{Consistency of Algorithm SC-GF}\label{sec_consistency}

\subsection{Deriving the Error Bound}

Once we get the approximate spectral embedding of the $n$
nodes of the graph in the form of $\widetilde{X}_R$,
we perform $k$-means with the rows of $\widetilde{X}_R$
as data points in $\mathbb{R}^{d}$. Let 
$c_1, \cdots, c_n \in \mathbb{R}^{d}$ be the centroids
corresponding to the $n$ rows of $\widetilde{X}_R$, out of which
only $k$ are unique. The $k$ unique centroids correspond to the 
centers of the $k$ clusters. Note that the true cluster centers
correspond to the rows of $\mathcal{X}_R$, and Theorem
\ref{Th_compressedconvSep} ensures that they are separable from
each other. Hence, we say that a node $i$ is correctly clustered
if its $k$-means cluster center $c_i$ is closer to its true cluster
center $\mathcal{X}_{R_{i*}}$ than it is to any other center
$\mathcal{X}_{R_{j*}}$, for $j \neq i$. In the following Lemma,
we lay down the sufficient condition for correctly clustering a node $i$.

\begin{lemma}[Sufficient Condition for Correct Clustering]\label{Lem_sufCondClus}
Let $c_1^{(n)}, \cdots, c_n^{(n)} \in \mathbb{R}^{d_n}$ be
the centroids resulting from performing $k_n$-means on the rows of $\widetilde{X}_R^{(n)}$.
For $P_n$ and $\epsilon_1$ as defined in Theorem \ref{Th_compressedconvSep},
\begin{align*}
\lVert c_i^{(n)} - \mathcal{X}_{R_{i*}}^{(n)} \rVert_2 < (1-\epsilon_1)\frac{1}{\sqrt{2P_n}} 
\Rightarrow
\lVert c_i^{(n)} - \mathcal{X}_{R_{i*}}^{(n)} \rVert_2 <
\lVert c_i^{(n)} - \mathcal{X}_{R_{j*}}^{(n)} \rVert_2.
\end{align*}
for any $z_i \neq z_j$.
\end{lemma}
\begin{proof}
See Appendix \ref{appC}.
\end{proof}
Following the analysis in \citep{rohe2011spectral}, we define
the set of misclustered vertices $\mathscr{M}$ as containing
the vertices that do not satisfy the sufficient condition in 
Lemma \ref{Lem_sufCondClus}.
\begin{align*}
\mathscr{M} = \Big\{i: \lVert c_i^{(n)} - \mathcal{X}_{R_{i*}}^{(n)} \rVert_2 \geq (1-\epsilon_1)\frac{1}{\sqrt{2P_n}} \Big\}
\end{align*}
Now that we have the definition for misclustered vertices,
we analyze the performance of $k$-means. Let the matrix
$C \in \mathbb{R}^{n \times k}$ be the result of $k$-means clustering
where the $i$th row, $c_i$ is the centroid corresponding to
the $i$th vertex. $C \in \mathscr{C}_{n,k}$ where $\mathscr{C}_{n,k}$
represents the family of matrices with $n$ rows out of which
only $k$ are unique. $C$ can be defined as
\begin{align*}
C = \mathop{\mathrm{arg} \min}_{M \in \mathscr{C}_{n,k}} \lVert M - \widetilde{X}_R \rVert_F^2.
\end{align*}

The next theorem bounds the number of misclustered vertices, 
that is the size of the set $\mathscr{M}$.
\begin{theorem}[Bound on the number of Misclustered Vertices]\label{Th_miscl}

\begin{align}\label{eq_miscl_big}
|\mathscr{M}| = o\bigg(P_n \big( \frac{(\log n)^2}{n \bar{\lambda}_{k_n}^4 \tau_n^4} + n^2 e_n^2 \big) \bigg)
\end{align}

\end{theorem}

\begin{proof}
See Appendix \ref{appC}.
\end{proof}
\subsection{Consistency in Special Cases}\label{Sec_SplCases}

We consider a simplified SBM with four parameters
$k$, $q$, $r$ and $s$ with $k$ blocks each of which contains $s$ nodes so that
the total number of vertices in the graph, $n = ks$. The probability of an edge between
two vertices of the same block is given by $q+r \in [0,1]$ and that of different blocks 
is given by $r \in [0,1]$. For the simplified SBM, the population of the largest block,
$P_n = s$. The smallest non-zero eigenvalue of the sampled graph Laplacian $L$ is given by
$\bar{\lambda}_{k_n} = \frac{1}{k(r/q)+1}$
and the parameter $\tau_n = q/k + r$ \citep{rohe2011spectral}. The proportion of the 
misclustered vertices is given by
\begin{align}\label{eq_miscl}
\frac{|\mathscr{M}|}{n} = o\Big( \frac{k^3}{n}(\log n)^2 + \frac{n^2}{k}e_n^2 \Big).
\end{align}
For weak consistency, we need $\lim_{n\to\infty} \frac{|\mathscr{M}|}{n} = 0$.
From \eqref{eq_miscl}, the condition on the number of clusters for weak consistency is $k = o(n^{1/3}/(\log n)^{2/3})$ and the worst case condition on the polynomial approximation error is $e_n = o(n^{-5/6}(\log n)^{1/3})$.

\section{Experiments}

We perform experiments on the simplified four parameter SBM presented
in Section \ref{Sec_SplCases}. For polynomial approximation of the ideal filter,
we use Chebyshev polynomials with Jackson damping coefficients \citep{di2016efficient}.

In our first experiment, we analyze the error rate for Algorithm~\ref{SC-GF_alg} for fixed number
of clusters as the number of nodes is increased. As expected, the proportion of misclustered
vertices, $\frac{|\mathscr{M}|}{n}$ tends to zero as $n$ grows large. However, for the case 
of high polynomial error ($p=5$) we see that the error rate diverges. This validates
the presence of $e_n$ in \eqref{eq_miscl}.

In our second experiment, we analyze the effect of the polynomial error $e_n$ in finer detail, by fixing all the
other variables, $n$, $k$, $q$ and $r$. From \eqref{eq_miscl} the proportion of misclustered vertices
should grow linearly with the squared polynomial error $e_n^2$. From Figure \ref{fig_polylin}, this 
behavior is evident.

\begin{figure}
\centering
    \includegraphics[width=.7\textwidth]{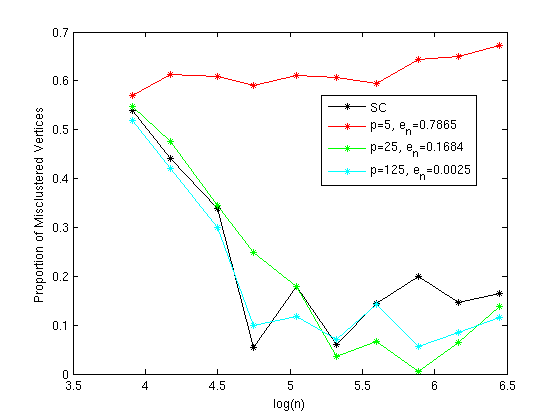}
\caption{Proportion of misclustered vertices plotted against the number of vertices.
$q=0.3$ and $r=0.1$. The polynomial order $p$ is set to $5$, $25$ and $125$ for the three curves pertaining
to Algorithm~\ref{SC-GF_alg}. The corresponding polynomial error $e_n$ is shown in the legend.}
\label{fig_miscl2}
\end{figure}
\begin{figure}
\centering
    \includegraphics[width=.7\textwidth]{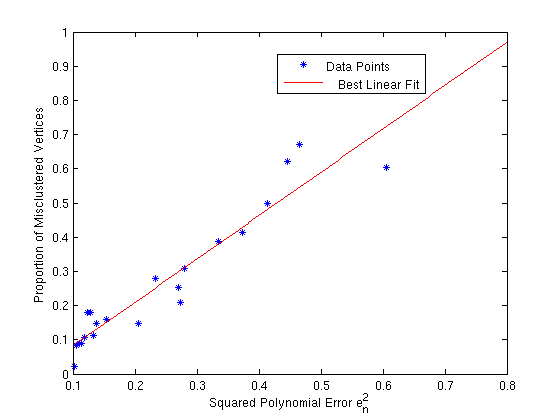}
\caption{Proportion of misclustered vertices plotted against the squared polynomial error, $e_n^2$.
$q=0.3$ and $r=0.1$. The polynomial order $p$ is varied from $5$ to $25$ linearly.}
\label{fig_polylin}
\end{figure}

\section{Conclusion}
In this paper, we prove some basic theorems that provide the theoretical basis for
spectral clustering done via graph filtering.
By Theorem~\ref{Th_compressedconvSep}, we prove the fundamental conditions required for the
consistency of the spectral clustering algorithm via graph filtering, namely separability and convergence.
By Theorem~\ref{Th_miscl}, we have shown that the algorithm
can retrieve the planted clusters in a stochastic block model consistently, and derive a bound on the number of
misclustered vertices. Through Lemma~\ref{Lem_ApproxFil}
and Lemma~\ref{Lem_lambda_k}, we quantify the maximum tolerable filtering error for the algorithm to succeed.
We then validate our results by performing experiments on the simulated stochastic block model.

While the results we prove in this paper provide evidence for the weak consistency of 
Algorithm~\ref{SC-GF_alg} under the stochastic block model under certain assumptions on sparsity and separability,
several problems still remain open. First, the bound on the accuracy of the $\lambda_k$ estimate as given in
Lemma \ref{Lem_lambda_k} is derived in terms of the polynomial approximation error $e_n$.
However, it is not trivial to estimate the polynomial order required to achieve a specific 
absolute error ($L_1$ norm) even in case of popular choices like the Jackson-Chebyshev polynomials \cite{di2016efficient}. This results in complications in deriving explicit expressions for 
the algorithm's computational complexity.
It also remains to be seen if the algorithm remains consistent under
a milder assumption on the graph sparsity ($\tau_n$) as is the case with the original spectral clustering
algorithm \cite{lei2015consistency}. While it is inevitable that the approximations
involved in estimating $\lambda_k$ (Lemma \ref{Lem_lambda_k}) and in obtaining the 
approximate spectral embedding (Lemma \ref{Lem_ApproxFil}) will result in a weaker bound on the performance,
we do not know if the results we derived are optimal. With this work, we hope to see a renewed interest
in graph filtering approaches to spectral algorithms which promise significant speed-ups in computation
while (provably) maintaining almost the same performance.

\clearpage
\vskip 0.2in
\bibliography{paper}

\begin{thebibliography}{20}
\providecommand{\natexlab}[1]{#1}
\providecommand{\url}[1]{\texttt{#1}}
\expandafter\ifx\csname urlstyle\endcsname\relax
  \providecommand{\doi}[1]{doi: #1}\else
  \providecommand{\doi}{doi: \begingroup \urlstyle{rm}\Url}\fi

\bibitem[Achlioptas(2003)]{achlioptas2003database}
Dimitris Achlioptas.
\newblock Database-friendly random projections: Johnson-lindenstrauss with
  binary coins.
\newblock \emph{Journal of computer and System Sciences}, 66\penalty0
  (4):\penalty0 671--687, 2003.

\bibitem[Choromanska et~al.(2013)Choromanska, Jebara, Kim, Mohan, and
  Monteleoni]{choromanska2013fast}
Anna Choromanska, Tony Jebara, Hyungtae Kim, Mahesh Mohan, and Claire
  Monteleoni.
\newblock Fast spectral clustering via the nystr{\"o}m method.
\newblock In \emph{International Conference on Algorithmic Learning Theory},
  pages 367--381. Springer, 2013.

\bibitem[Di~Napoli et~al.(2016)Di~Napoli, Polizzi, and Saad]{di2016efficient}
Edoardo Di~Napoli, Eric Polizzi, and Yousef Saad.
\newblock Efficient estimation of eigenvalue counts in an interval.
\newblock \emph{Numerical Linear Algebra with Applications}, 23\penalty0
  (4):\penalty0 674--692, 2016.

\bibitem[Fortunato(2010)]{fortunato2010community}
Santo Fortunato.
\newblock Community detection in graphs.
\newblock \emph{Physics reports}, 486\penalty0 (3):\penalty0 75--174, 2010.

\bibitem[Fowlkes et~al.(2004)Fowlkes, Belongie, Chung, and
  Malik]{fowlkes2004spectral}
Charless Fowlkes, Serge Belongie, Fan Chung, and Jitendra Malik.
\newblock Spectral grouping using the nystrom method.
\newblock \emph{IEEE transactions on Pattern Analysis and Machine
  Intelligence}, 26\penalty0 (2):\penalty0 214--225, 2004.

\bibitem[Gittens et~al.(2013)Gittens, Kambadur, and
  Boutsidis]{gittens2013approximate}
Alex Gittens, Prabhanjan Kambadur, and Christos Boutsidis.
\newblock Approximate spectral clustering via randomized sketching.
\newblock \emph{Ebay/IBM Research Technical Report}, 2013.

\bibitem[Holland et~al.(1983)Holland, Laskey, and
  Leinhardt]{holland1983stochastic}
Paul~W Holland, Kathryn~Blackmond Laskey, and Samuel Leinhardt.
\newblock Stochastic blockmodels: First steps.
\newblock \emph{Social networks}, 5\penalty0 (2):\penalty0 109--137, 1983.

\bibitem[Jain et~al.(1999)Jain, Murty, and Flynn]{jain1999data}
Anil~K Jain, M~Narasimha Murty, and Patrick~J Flynn.
\newblock Data clustering: a review.
\newblock \emph{ACM computing surveys (CSUR)}, 31\penalty0 (3):\penalty0
  264--323, 1999.

\bibitem[Lei et~al.(2015)Lei, Rinaldo, et~al.]{lei2015consistency}
Jing Lei, Alessandro Rinaldo, et~al.
\newblock Consistency of spectral clustering in stochastic block models.
\newblock \emph{The Annals of Statistics}, 43\penalty0 (1):\penalty0 215--237,
  2015.

\bibitem[Li et~al.(2011)Li, Lian, Kwok, and Lu]{Li2011TimeAS}
Mu~Li, Xiao-Chen Lian, James~T. Kwok, and Bao-Liang Lu.
\newblock Time and space efficient spectral clustering via column sampling.
\newblock In \emph{CVPR}, 2011.

\bibitem[Ng et~al.(2002)Ng, Jordan, Weiss, et~al.]{ng2002spectral}
Andrew~Y Ng, Michael~I Jordan, Yair Weiss, et~al.
\newblock On spectral clustering: Analysis and an algorithm.
\newblock \emph{Advances in neural information processing systems}, 2:\penalty0
  849--856, 2002.

\bibitem[Puy et~al.(2016)Puy, Tremblay, Gribonval, and
  Vandergheynst]{puy2016random}
Gilles Puy, Nicolas Tremblay, R{\'e}mi Gribonval, and Pierre Vandergheynst.
\newblock Random sampling of bandlimited signals on graphs.
\newblock \emph{Applied and Computational Harmonic Analysis}, 2016.

\bibitem[Ramasamy and Madhow(2015)]{ramasamy2015compressive}
Dinesh Ramasamy and Upamanyu Madhow.
\newblock Compressive spectral embedding: sidestepping the svd.
\newblock In \emph{Advances in Neural Information Processing Systems}, pages
  550--558, 2015.

\bibitem[Rohe et~al.(2011)Rohe, Chatterjee, and Yu]{rohe2011spectral}
Karl Rohe, Sourav Chatterjee, and Bin Yu.
\newblock Spectral clustering and the high-dimensional stochastic blockmodel.
\newblock \emph{The Annals of Statistics}, pages 1878--1915, 2011.

\bibitem[Sakai and Imiya(2009)]{sakai2009fast}
Tomoya Sakai and Atsushi Imiya.
\newblock Fast spectral clustering with random projection and sampling.
\newblock In \emph{International Workshop on Machine Learning and Data Mining
  in Pattern Recognition}, pages 372--384. Springer, 2009.

\bibitem[Shi and Malik(2000)]{shi2000normalized}
Jianbo Shi and Jitendra Malik.
\newblock Normalized cuts and image segmentation.
\newblock \emph{IEEE Transactions on Pattern Analysis and Machine
  Intelligence}, 22\penalty0 (8):\penalty0 888--905, 2000.

\bibitem[Shuman et~al.(2013)Shuman, Narang, Frossard, Ortega, and
  Vandergheynst]{shuman2013emerging}
David~I Shuman, Sunil~K Narang, Pascal Frossard, Antonio Ortega, and Pierre
  Vandergheynst.
\newblock The emerging field of signal processing on graphs: Extending
  high-dimensional data analysis to networks and other irregular domains.
\newblock \emph{IEEE Signal Processing Magazine}, 30\penalty0 (3):\penalty0
  83--98, 2013.

\bibitem[Tremblay et~al.(2016{\natexlab{a}})Tremblay, Puy, Borgnat,
  Vandergheynst, et~al.]{tremblay2016accelerated}
Nicolas Tremblay, Gilles Puy, Pierre Borgnat, Pierre Vandergheynst, et~al.
\newblock Accelerated spectral clustering using graph filtering of random
  signals.
\newblock In \emph{2016 IEEE International Conference on Acoustics, Speech and
  Signal Processing (ICASSP)}, pages 4094--4098. IEEE, 2016{\natexlab{a}}.

\bibitem[Tremblay et~al.(2016{\natexlab{b}})Tremblay, Puy, Gribonval, and
  Vandergheynst]{tremblay2016compressive}
Nicolas Tremblay, Gilles Puy, R{\'e}mi Gribonval, and Pierre Vandergheynst.
\newblock Compressive spectral clustering.
\newblock In \emph{Machine Learning, Proceedings of the Thirty-third
  International Conference (ICML 2016), June}, pages 20--22,
  2016{\natexlab{b}}.

\bibitem[Von~Luxburg(2007)]{von2007tutorial}
Ulrike Von~Luxburg.
\newblock A tutorial on spectral clustering.
\newblock \emph{Statistics and computing}, 17\penalty0 (4):\penalty0 395--416,
  2007.

\end{thebibliography}

\appendix

\section{Proofs for Theorems in Section~\ref{sec_SCviaGF}}
\label{appB}

\subsection{Proof of Theorem \ref{Th_compressedconvSep}}
\begin{proof}
For the sake of compactness, we omit the superscript $(n)$
for the sequences of matrices, as the analysis is valid at
every $n$.

By Theorem \ref{Th_sep}, there are at most $k$ unique rows
out of the $n$ rows of the matrix $\mathcal{X}$, while the
$n$ rows of the matrix $X$, can potentially be unique. 
The same inference can be made for the matrices 
$\mathcal{X}OX^T$ and $XX^T$, where $O$ is the 
orthonormal matrix from Theorem \ref{Th_conv}.

Treating the combined $n+k$ unique rows of the two matrices 
as data points in $\mathbb{R}^k$, we can use the 
Johnson-Lindenstrauss Lemma to approximately preserve 
the pairwise Euclidian distances between any two rows 
up to a factor of $\epsilon_1$. Applying Theorem 1.1 from 
\citet{achlioptas2003database}, if $d_n$ is larger than
\begin{equation*}
 \frac{4+2\beta}{\epsilon_1^2/2-\epsilon_1^3/3}\log (n+k),
\end{equation*}
then with probability at least $1-n^{-\beta}$, we have
\begin{align}\label{eq_jl_chi}
(1-\epsilon_1) 
\lVert {\mathcal{X}}_{i*}OX^T -{\mathcal{X}}_{j*}OX^T \rVert_2 
\leq 
\lVert {\mathcal{X}_R}_{i*} -{\mathcal{X}_R}_{j*} \rVert_2
\leq 
(1+\epsilon_1) 
\lVert {\mathcal{X}}_{i*}OX^T -{\mathcal{X}}_{j*}OX^T \rVert_2  
\end{align}
for any $Z_{i*} \neq Z_{j*}$,
\begin{align*}
(1-\epsilon_1) 
\lVert X_{i*}X^T -X_{j*}X^T \rVert_2 
\leq 
\lVert {\mathcal{X}_R}_{i*} -{\mathcal{X}_R}_{j*} \rVert_2 
\leq 
(1+\epsilon_1) 
\lVert X_{i*}X^T -X_{j*}X^T \rVert_2 
\end{align*}
and
\begin{align}\label{eq_jl_X}
(1-\epsilon_1) 
\lVert {{X}}_{i*}OX^T - {\mathcal{X}}_{j*}X^T \rVert_2 
\leq 
\lVert {{X}_R}_{i*} - {\mathcal{X}_R}_{j*} \rVert_2 
\leq 
(1+\epsilon_1) 
\lVert {{X}}_{i*}OX^T - {\mathcal{X}}_{j*}X^T \rVert_2 
\end{align}
where $i,j \in \{1, \cdots, n\}$.

Combining the inequality on the let side of \eqref{eq_jl_chi} 
with Statement 3 of Theorem \ref{Th_sep}, we get
\begin{align*}
\lVert {\mathcal{X}_R}_{i*} -{\mathcal{X}_R}_{j*} \rVert_2 
&\geq 
(1-\epsilon_1) 
\lVert {\mathcal{X}}_{i*}OX^T -{\mathcal{X}}_{j*}OX^T \rVert_2 \\
&=
(1-\epsilon_1)
\lVert ({\mathcal{X}}_{i*} -{\mathcal{X}}_{j*})OX^T \rVert_2 \\
&=
(1-\epsilon_1)
\lVert {\mathcal{X}}_{i*} -{\mathcal{X}}_{j*} \rVert_2 \\
&\geq (1-\epsilon)\sqrt{2/P_n}
\end{align*}
for any $Z_{i*} \neq Z_{j*}$. Since $X^T X$ is an identity matrix,
the rows of $OX^T$ are orthogonal. Hence, multiplication of a vector
by $OX^T$ from the right does not change the norm. By a similar
procedure, combining the inequality on the right side of 
\eqref{eq_jl_X} with Theorem \ref{Th_conv}, we get
\begin{align*}
\lVert X_R - \mathcal{X}_R \rVert_F^2
=
\sum_{i=1}^n \lVert {{X}_R}_{i*} - {\mathcal{X}_R}_{i*} \rVert_2^2 
&\leq
(1+\epsilon_1)^2
\sum_{i=1}^n 
\lVert {{X}}_{i*}OX^T - {\mathcal{X}}_{j*}X^T \rVert_2^2 \\
&=
(1+\epsilon_1)^2
\sum_{i=1}^n 
\lVert ({{X}}_{i*}O - {\mathcal{X}}_{j*})X^T \rVert_2^2 \\
&=
(1+\epsilon_1)^2
\sum_{i=1}^n 
\lVert {{X}}_{i*}O - {\mathcal{X}}_{j*} \rVert_2^2 \\
&=
(1+\epsilon_1)^2
\lVert {{X}}_{i*}O - {\mathcal{X}}_{j*} \rVert_F^2 \\
&=
o\bigg(\frac{(\log n)^2}{n \bar{\lambda}_{k_n}^4 \tau_n^4}\bigg)
\end{align*}
\end{proof}
\subsection{Proof of Lemma \ref{Lem_ApproxFil}}
\begin{proof}
Firstly, we note that $\lVert R^{(n)} \rVert_F^2$ is a chi-squared
random variable with $nd_n$ degrees of freedom and mean $n$.
Using the Chernoff bound on $\lVert R^{(n)} \rVert_F^2$, we have
\begin{align}\label{eq_RnProb}
\mathsf{Pr} \Big( \Big| \frac{1}{n} \lVert R^{(n)} \rVert_F^2 - 1 \Big| > \epsilon_2 \Big) \leq e^{-n d_n(\epsilon_2^2 - \epsilon_2^3)/4}
\end{align}
Now to bound the difference between the ideal and polynomial filters,
\begin{align}\label{eq_filtErr}
\lVert U (\widetilde{h}_{\lambda_{k_n}}(\Lambda) - h_{\lambda_{k_n}}(\Lambda)) U^T \rVert_F^2
=
\lVert \widetilde{h}_{\lambda_{k_n}}(\Lambda) - h_{\lambda_{k_n}}(\Lambda) \rVert_F^2
&=
\sum_{i=1}^n ( \widetilde{h}_{\lambda_{k_n}}(\lambda_i) - h_{\lambda_{k_n}}(\lambda_i) )^2 \nonumber \\
&\leq
\sum_{i=1}^n e_n^2
=
ne_n^2.
\end{align}
Using the result from \eqref{eq_RnProb} and \eqref{eq_filtErr}, 
we can bound the difference between the ideal and approximate
spectral embedding as follows.
\begin{align*}
\lVert \widetilde{X}_R^{(n)} - X_R^{(n)} \rVert_F^2 
=
\lVert \widetilde{H}_{\lambda_{k_n}} R - H_{\lambda_{k_n}} R^{(n)} \rVert_F^2
&=
\lVert U (\widetilde{h}_{\lambda_{k_n}}(\Lambda) - h_{\lambda_{k_n}}(\Lambda)) U^T R^{(n)} \rVert_F^2 \\
&\leq
\lVert U (\widetilde{h}_{\lambda_{k_n}}(\Lambda) - h_{\lambda_{k_n}}(\Lambda)) U^T \rVert_F^2
\lVert R^{(n)} \rVert_F^2 \\
&\leq
(1+ \epsilon_2) n^2 e_n^2
\end{align*}
where the last step follows with a probability of
 at least $1 - e^{-n d_n(\epsilon_2^2 - \epsilon_2^3)/4}$.
\end{proof}
\subsection{Proof of Lemma \ref{Lem_lambda_k}}
\begin{proof}
For the sake of compactness, we omit the superscript $(n)$
for the sequences of matrices, as the analysis is valid at
every $n$.

From Lemma \ref{Lem_ApproxFil}, we have a bound on 
the term $\lVert \widetilde{X}_R - X_R \rVert_F^2$.
So, we proceed to prove Lemma \ref{Lem_lambda_k} by bounding
the term $\lVert X_R \rVert_F^2$. For this, we make use 
of the fact that the $k_n$ columns of $X^{(n)}$ are orthonormal.
\begin{align}\label{eq_XR}
\lVert X_R \rVert_F^2 
= \lVert X {X^{(n)}}^T R \rVert_F^2
= k_n \lVert R \rVert_F^2
\end{align}
Combining \eqref{eq_XR} with \eqref{eq_RnProb},
we have the following with probability exceeding 
$1 - e^{-n d_n(\epsilon_2^2 - \epsilon_2^3)/4}$.
\begin{align*}
(1- \epsilon_2)k_n 
\leq 
\frac{1}{n} \lVert X_R \rVert_F^2 
\leq
(1+ \epsilon_2)k_n.
\end{align*}
Now, we prove the upper bound on $\widetilde{X}_R$.
\begin{align}\label{eq_lam1}
\lVert \widetilde{X}_R \rVert_F^2
= 
\mathop{tr} \big( \widetilde{X}_R^T \widetilde{X}_R \big)
&=
\mathop{tr} \big( R^T U \widetilde{h}_{\lambda_k}(\Lambda) U^T U \widetilde{h}_{\lambda_k}(\Lambda) U^T R \big)
\nonumber \\
&=
\mathop{tr} \big( R^T U (\widetilde{h}_{\lambda_k}(\Lambda))^2 U^T R \big) \nonumber \\
&=
\mathop{tr} \big((\widetilde{h}_{\lambda_k}(\Lambda))^2 U^T R R^T U \big) \nonumber \\
&\leq
\mathop{tr} \big((\widetilde{h}_{\lambda_k}(\Lambda))^2 \big)  \mathop{tr} \big( U^T R R^T U \big)
\end{align}
where the last statement follows from the fact that the matrices $\widetilde{X}_R^T \widetilde{X}_R$,
$(\widetilde{h}_{\lambda_k}(\Lambda))^2$ and $U^T R R^T U$ are non-negative semi-definite.
\begin{align}\label{eq_lam2}
\mathop{tr} \big( U^T R R^T U \big) &= \mathop{tr} \big(R R^T\big)
= \lVert R \rVert_F^2 \leq (1+\epsilon_2)n.
\end{align}
The last statement follows from \eqref{eq_RnProb} with a probability of
at least $1 - e^{-n d_n(\epsilon_2^2 - \epsilon_2^3)/4}$.

Using the definition of the maximum filter error $e_n$, we get
\begin{align}\label{eq_lam3}
\mathop{tr} \big((\widetilde{h}_{\lambda_k}(\Lambda))^2 \big)
\leq k(1+e_n)^2 + (n-k)e_n^2
&= k + 2ke_n + ne_n^2.
\end{align}

Combining \eqref{eq_lam2} and \eqref{eq_lam3} with \eqref{eq_lam1}, we get
\begin{align}\label{eq_lam_lb}
\frac{1}{n} \lVert \widetilde{X}_R \rVert_F^2 \leq (1+\epsilon_2)(k + 2ke_n + ne_n^2).
\end{align}
Now we proceed to proving the lower bound.
\begin{align}\label{eq_lam4}
\lVert \widetilde{X}_R \rVert_F^2
=
\lVert X_R \rVert_F^2 + \lVert \widetilde{X}_R - X_R \rVert_F^2
 +2 \mathop{tr}(X_R^T(\widetilde{X}_R - X_R)).
\end{align}
\begin{align}\label{eq_lam5}
\mathop{tr}\big(X_R^T(\widetilde{X}_R - X_R)\big)
&=
\mathop{tr}\big(R^T U h_{\lambda_k}(\Lambda) U^T U (\widetilde{h}_{\lambda_k}(\Lambda) - h_{\lambda_k}(\Lambda)) U^T R\big)
\nonumber \\
&=
\mathop{tr}\big(h_{\lambda_k}(\Lambda)(\widetilde{h}_{\lambda_k}(\Lambda) - h_{\lambda_k}(\Lambda)) U^T R R^T U\big)\nonumber \\
&=
\mathop{tr}\big(h_{\lambda_k}(\Lambda)(\widetilde{h}_{\lambda_k}(\Lambda) - h_{\lambda_k}(\Lambda) + e_n I_n) U^T R R^T U\big)  \nonumber \\
&\ \ \ - \mathop{tr} \big( h_{\lambda_k}(\Lambda) e_n I_n U^T R R^T  \big).
\end{align}
Here, $I_n \in \mathbb{R}^{n \times n}$ is the Identity matrix.
By the definition of $e_n$, the diagonal entries of 
$(\widetilde{h}_{\lambda_k}(\Lambda) - h_{\lambda_k}(\Lambda) + e_n I_n)$
are non-negative. Hence, the first term in \eqref{eq_lam5} is non-negative.
For the second term we have,
\begin{align}\label{eq_lam6}
\mathop{tr} \big( h_{\lambda_k}(\Lambda) e_n I_n U^T R R^T  \big)
\leq
e_n \mathop{tr} \big( h_{\lambda_k}(\Lambda)\big) \mathop{tr} \big( U^T R R^T  \big)
=
e_n k_n (1+\epsilon_2)n.
\end{align}
In addition, the term $\lVert \widetilde{X}_R - X_R \rVert_F^2$ in \eqref{eq_lam4} is non-negative.
Combining \eqref{eq_lam6} with \eqref{eq_lam4}, we get
\begin{align}\label{eq_lam_ub}
\frac{1}{n}\lVert \widetilde{X}_R \rVert_F^2
&\geq
(1-\epsilon_2)k_n - 2(1+\epsilon_2)e_n k_n.
\end{align}

Putting together \eqref{eq_lam_lb} and \eqref{eq_lam_ub}, 
we prove Lemma \ref{Lem_lambda_k}.
\end{proof}

\section{Proofs for Theorems in Section~\ref{sec_consistency}}
\label{appC}

\subsection{Proof of Lemma \ref{Lem_sufCondClus}}
\begin{proof}
We follow a similar technique as that of Lemma 3.2 in \citet{rohe2011spectral}.
Suppose that $\lVert c_i^{(n)} - \mathcal{X}_{R_{i*}}^{(n)} \rVert_2 < (1-\epsilon_1)\frac{1}{\sqrt{2P_n}}$
for some $i$. For any $z_j \neq z_i$, we have
\begin{align*}
\lVert c_i^{(n)} - \mathcal{X}_{R_{j*}}^{(n)} \rVert_2 
\geq
\lVert \mathcal{X}_{R_{i*}}^{(n)} - \mathcal{X}_{R_{j*}}^{(n)} \rVert_2 -
\lVert c_i^{(n)} - \mathcal{X}_{R_{i*}}^{(n)} \rVert_2 
&\geq
(1-\epsilon_1)\sqrt{\frac{2}{P_n}} - (1-\epsilon_1)\frac{1}{\sqrt{2P_n}} \\
&= (1-\epsilon_1)\frac{1}{\sqrt{2P_n}}
\end{align*}
Here we have used the result of Theorem \ref{Th_compressedconvSep}
on the separability of the rows of $\mathcal{X}_{R}^{(n)}$.
\end{proof}

\subsection{Proof of Theorem \ref{Th_miscl}}
\begin{proof}
From the output of the $k$-means we have
\begin{align*}
C = \mathop{arg\ min}_{M \in \mathscr{C}_{n,k}} \lVert M - \widetilde{X}_R \rVert_F^2
\end{align*}
From Theorem \ref{Th_sep} we know that only $k$ rows of the matrix
$\mathcal{X}$ are unique out of its $n$ rows. The same inference can be
made about $\mathcal{X}_R$. Hence, $\mathcal{X}_R \in \mathscr{C}_{n,k}$.
By the optimality of $k$-means we have
\begin{align*}
\lVert C - \widetilde{X}_R \rVert_F^2
\leq
\lVert \mathcal{X}_R - \widetilde{X}_R \rVert_F^2 
&\leq
2\lVert \mathcal{X}_R - X_R \rVert_F^2 + 2\lVert X_R - \widetilde{X}_R \rVert_F^2.
\end{align*}
Hence
\begin{align*}
\lVert C - \mathcal{X}_R \rVert_F^2 
&\leq
\lVert C - \widetilde{X}_R \rVert_F^2 + \lVert \widetilde{X}_R - \mathcal{X}_R \rVert_F^2 \\
&\leq
\Big( 2\lVert \mathcal{X}_R - X_R \rVert_F^2 + 2\lVert X_R - \widetilde{X}_R \rVert_F^2 \Big) 
 + \Big( 2 \lVert \widetilde{X}_R - X_R \rVert_F^2 + 2 \lVert X_R - \mathcal{X}_R \rVert_F^2 \Big)\\
&=
4 \lVert \mathcal{X}_R - X_R \rVert_F^2 + 4 \lVert X_R - \widetilde{X}_R \rVert_F^2
\end{align*}
From the definition of the misclustered vertices,
\begin{align*}
|\mathscr{M}| \leq \sum_{i \in \mathscr{M}} 1 
&\leq
\frac{2P_n}{(1-\epsilon_1)^2} \sum_{i \in \mathscr{M}} \lVert c_i - \mathcal{X}_{R_{i*}} \rVert_F^2 \\
&\leq
\frac{2P_n}{(1-\epsilon_1)^2} \lVert C - \mathcal{X}_{R} \rVert_F^2 \\
&\leq
\frac{2P_n}{(1-\epsilon_1)^2}
\Big(4 \lVert \mathcal{X}_R - X_R \rVert_F^2 + 4 \lVert X_R - \widetilde{X}_R \rVert_F^2 \Big)\\
&=
o\bigg(P_n \big( \frac{(\log n)^2}{n \bar{\lambda}_{k_n}^4 \tau_n^4} + n^2 e_n^2 \big) \bigg).
\end{align*}
The last statement follows from Theorem \ref{Th_compressedconvSep} and Lemma \ref{Lem_ApproxFil}.
\end{proof}

\end{document}